\documentclass[12pt]{article}
\usepackage{multirow}
\usepackage[unicode=true]{hyperref}
\usepackage{breakurl}
\usepackage{pifont}
\usepackage{amsmath,amsfonts,stmaryrd,amssymb} 
\usepackage{ifthen}
\usepackage{verbatim}
\usepackage{graphicx}
\usepackage{color}
\PassOptionsToPackage{normalem}{ulem}
\usepackage{ulem}
\usepackage{url}

\usepackage{enumerate} 

\usepackage[ruled]{algorithm2e} 

\usepackage[framemethod=tikz]{mdframed} 

\usepackage{listings} 
\usepackage{geometry} 

\usepackage[utf8]{inputenc} 
\usepackage[T1]{fontenc} 

\usepackage{XCharter} 

\mdfdefinestyle{commandline}{
	leftmargin=10pt,
	rightmargin=10pt,
	innerleftmargin=15pt,
	middlelinecolor=black!50!white,
	middlelinewidth=2pt,
	frametitlerule=false,
	backgroundcolor=black!5!white,
	frametitle={Command Line},
	frametitlefont={\normalfont\sffamily\color{white}\hspace{-1em}},
	frametitlebackgroundcolor=black!50!white,
	nobreak,
}

\mdfdefinestyle{file}{
	innertopmargin=1.6\baselineskip,
	innerbottommargin=0.8\baselineskip,
	topline=false, bottomline=false,
	leftline=false, rightline=false,
	leftmargin=2cm,
	rightmargin=2cm,
	singleextra={%
		\draw[fill=black!10!white](P)++(0,-1.2em)rectangle(P-|O);
		\node[anchor=north west]
		at(P-|O){\ttfamily\mdfilename};
		\def\l{3em}
		\draw(O-|P)++(-\l,0)--++(\l,\l)--(P)--(P-|O)--(O)--cycle;
		\draw(O-|P)++(-\l,0)--++(0,\l)--++(\l,0);
	},
	nobreak,
}

\mdfdefinestyle{question}{
	innertopmargin=1.2\baselineskip,
	innerbottommargin=0.8\baselineskip,
	roundcorner=5pt,
	nobreak,
	singleextra={%
		\draw(P-|O)node[xshift=1em,anchor=west,fill=white,draw,rounded corners=5pt]{%
			Question \theQuestion\questionTitle};
	},
}

\newcounter{Question} 

\mdfdefinestyle{warning}{
	topline=false, bottomline=false,
	leftline=false, rightline=false,
	nobreak,
	singleextra={%
		\draw(P-|O)++(-0.5em,0)node(tmp1){};
		\draw(P-|O)++(0.5em,0)node(tmp2){};
		\fill[black,rotate around={45:(P-|O)}](tmp1)rectangle(tmp2);
		\node at(P-|O){\color{white}\scriptsize\bf !};
		\draw[very thick](P-|O)++(0,-1em)--(O);
	}
}

\mdfdefinestyle{info}{%
	topline=false, bottomline=false,
	leftline=false, rightline=false,
	nobreak,
	singleextra={%
		\fill[black](P-|O)circle[radius=0.4em];
		\node at(P-|O){\color{white}\scriptsize\bf i};
		\draw[very thick](P-|O)++(0,-0.8em)--(O);
	}
}

\usepackage{setspace}

\usepackage{graphicx}
\usepackage{amssymb}
\usepackage{multirow}
\usepackage{verbatim}

\usepackage{amsthm}
\newtheorem{thm}{Claim}
\newtheorem{cor}{Corollary}

\usepackage{amsmath}
\usepackage{bm}
\usepackage{authblk}

\usepackage{tabularx}
\usepackage{xcolor}

\begin{document}


\title{Variable selection with false discovery rate control in deep neural networks}



\renewcommand\Authfont{\fontsize{12}{14.4}\selectfont}
\renewcommand\Affilfont{\fontsize{9}{10.8}\itshape}
\author[1]{Zixuan Song}
\author[1,*]{Jun Li}
\affil[1]{Department of Applied and Computational Mathematics and Statistics, University of Notre Dame, Notre Dame, IN 46556, USA}
\affil[*]{To whom correspondence should be addressed. Tel: +1 574 631 3429; Fax: +1 574 631 4822; Email: jun.li@nd.edu}

\date{}
\maketitle


\doublespacing

\begin{abstract}
Deep neural networks (DNNs) are famous for their high prediction accuracy, but they are also known for their black-box nature and poor interpretability. We consider the problem of variable selection, that is, selecting the input variables that have significant predictive power on the output, in DNNs. We propose a backward elimination procedure called SurvNet, which is based on a new measure of variable importance that applies to a wide variety of networks. More importantly, SurvNet is able to estimate and control the false discovery rate of selected variables, while no existing methods provide such a quality control. Further, SurvNet adaptively determines how many variables to eliminate at each step in order to maximize the selection efficiency. To study its validity, SurvNet is applied to image data and gene expression data, as well as various simulation datasets.

\end{abstract}

Deep neural networks (DNNs) are a popular machine learning technique and have shown superior performance in many scientific problems. Despite of their high prediction accuracy, DNNs are often criticized for a lack of interpretation of how changes of the input variables influence the output. Indeed, for applications in many scientific fields such as biology and medicine, understanding the statistical models described by the networks can be as important as, if not more important than, the prediction accuracy. In a DNN, because of its nonlinearity and inherent complexity, generally one should not expect a concise relationship between each input variable and the output, such as the conditional monotonicity in linear regression or logistic regression. A more realistic approach for interpreting the DNN model can be selecting a subset of variables, among all input variables, that have significant predictive power on the output, which is known as ``variable selection''. This paper considers the variable selection problem in DNNs.

During the past decades, many methods have been proposed for this task. The variable selection methods for neural networks, similar to the ones for other machine learning techniques, can be broadly classified into three categories: filters, wrappers and embedded methods \cite{may_review_2011-1,guyon_introduction_2003,chandrashekar_survey_2014}. Filters select variables by information theoretic criteria such as mutual information \cite{battiti_using_1994} and partial mutual information \cite{may_non-linear_2008}, and the selection procedure does not involve network training. In contrast, both wrappers and embedded methods are based on the training of neural networks. Wrappers wrap the training phase with a search strategy, which searches through the set, or a subset, of all possible combinations of input variables and selects the combination whose corresponding network gives the highest prediction accuracy. A number of sequential \cite{sung1998ranking,maier1998use} and heuristic search strategies \cite{brill_fast_1992,tong_genetic_2010,sivagaminathan_hybrid_2007} have been used. Embedded methods, unlike wrappers, select variables during the training of the network of interest. This can be done by gradually removing/pruning weights or variables according to their importance measured in various ways (a detailed review is given in the Methods section) or by incorporating a regularization term into the loss function of the neural network to impose sparsity on the weights \cite{grandvalet_outcomes_1999,chapados_input_2001,simila_combined_2009,scardapane_group_2017}. For a more exhaustive review of variable selection methods in neural networks, see \cite{may_review_2011-1,zhang_neural_2000}.

While a lot of variable selection methods have been developed for neural networks, there are still challenges that hinder them from being widely used. First and foremost, these methods lack a control on the quality of selected variables. When selecting from a large number of variables, a standard way of quality control is to calculate false discovery rate (FDR) \cite{benjamini1995controlling} and control it at a certain level, particularly in biological and medical studies. In the context of variable selection, FDR is the (expected) proportion of false positives among all variables called significant; for example, if 20 variables are selected (called significant), and two of them are actually null, then the FDR is $2/20=0.1$. However, no variable selection methods for neural networks so far have tried to estimate FDR or keep FDR under control. Second, among these methods, many were developed for specific types of networks, especially very shallow networks, and they do not work, or work inefficiently, for deeper networks. Third, many of the methods are not applicable to large datasets, on which their computational loads can be prohibitively high.

In this paper, we develop a method called SurvNet for variable selection in neural networks that overcomes these limitations. It is an embedded method that gradually removes least relevant variables until the FDR of remaining variables reaches a desired threshold. Figure \ref{fig:1} is the flowchart of SurvNet. It starts by adding a set of simulated input variables called ``surrogate variables'' that will help estimate FDR and training a network with all variables, including both original and surrogate variables. Then it calculates the importance of each variable (original or surrogate) and eliminates the variables that are least important. When eliminating a variable, its corresponding input neuron and all outgoing connections of this neuron are removed from the network. After this, SurvNet estimates the FDR of the original variables that remain in the model. If the estimated FDR is greater than the pre-set threshold, SurvNet will go back to the step of training the (updated) network; otherwise, the elimination stops, and all remaining surrogate variables are removed before the final model is trained. Note that each updated network is trained using the values of weights in the last trained network as initial values for a ``warm start''.

There are three major novelties in this backward elimination procedure of SurvNet. First, it proposes a new measure/score of variable importance, which works regardless of the type of problems (classification or regression), the number of output neurons (one or multiple), and the number of hidden layers (one or multiple) in neural networks. In fact, this score can be readily computed for networks with arbitrary depths and activation functions. Second, SurvNet proposes an easy and quick way of estimating FDRs. Statistical estimation of FDRs requires obtaining the null distribution of the importance scores, that is, the distribution of the scores of irrelevant variables \cite{storey2003statistical}. This is often done by permuting the output values of samples and training multiple independent models in parallel, each of which corresponds to a permuted dataset, but the computational cost is typically unaffordable for neural networks. SurvNet proposes a distinct way. It generates a set of null variables which serve as surrogates of the (unknown) null original variables to obtain the null distribution. With the introduction of surrogate variables, an estimate of FDR can be given by a simple mathematical formula without training a large number of networks at each step. Third, instead of eliminating one variable or any pre-specified number of variables at each step, SurvNet is able to adaptively determine an appropriate number of variables to eliminate by itself. This number, expressed in a concise mathematical formula, makes the elimination highly efficient while having the FDR well controlled on the desired level. The formula includes a parameter called ``elimination rate'', which is a constant between 0 and 1 and controls the ``aggressiveness'' of elimination. When this parameter is chosen to be 1, the elimination is the most aggressive.

Put together, SurvNet is a computationally efficient mechanism for variable selection in neural networks that needs little manual intervention. After setting the initial network structure, an FDR cutoff $\eta ^*$ (0.1 is the most commonly used value), and an elimination rate $\varepsilon$ (1 is often an acceptable choice), the elimination procedure will automatically determine how many and which variables to eliminate at each step and stop when the estimated FDR is no greater than $\eta ^*$.

\section*{Data and results}
We applied SurvNet to digits 4's and 9's in the MNIST database (Dataset 5), a single-cell RNA-Seq dataset (Dataset 6), as well as four simulation datasets (Datasets 1 $\sim$ 4).

MNIST \cite{lecun1998gradient} contains 60,000 training images (including 5,000 validation images) and 10,000 testing images of ten handwritten digits from 0 to 9. Each image contains $28\times28 = 784$ pixels, which are treated as 784 input variables.

Single-cell RNA-Seq \cite{kolodziejczyk2015technology} is a biological technique for measuring gene expression in cells. Along with other single-cell techniques, it was recognized as the ``2018 Breakthrough of the Year'' by the \textit{Science} magazine on account of its important applications in biomedical and genomic research. In single-cell RNA-Seq data, the samples are the cells, the inputs are expression levels of individual genes, and the output is the cell type. Biologically, it is often believed that the cell type is determined by a small set of genes, and thus single-cell RNA-Seq data can be a good choice to study variable selection. 

The classification accuracy of SurvNet for these real data was evaluated by several criteria, including initial test loss, initial test error, final test loss and final test error. Here ``test loss'' and ``test error'' refer to the cross-entropy loss and the misclassification rate on the test data, respectively; and their ``initial'' and ``final'' values were derived by using the network with all original variables and with selected variables only, respectively. See Supplementary Materials for details about how they were calculated.

For these real datasets (Datasets 5 $\sim$ 6), however, it is unknown which variables are truly significant. Hence we relied on simulated data to quantitatively assess the accuracy of selected variables, the most important aspect of SurvNet. Four datasets were simulated under different schemes. Datasets 1 $\sim$ 3 were for classification and Dataset 4 was for regression.

Except for the MNIST data, each dataset was divided into a training set and a test set, with 80\% of the samples in the training set and 20\% in the test set, and 30\% of training samples were further separated for validation (used to decide when to stop training, see Supplementary Materials).

SurvNet was implemented on TensorFlow 1.8 \cite{tensorflow2015-whitepaper}. For each dataset, we used a common and simple network structure with two hidden layers, which consisted of 40 and 20 nodes respectively. The ReLU activation function was used, together with a batch size of 50 and a learning rate of 0.05 (0.01 for the regression problem).

In our experiments, Datasets 1 $\sim$ 4 were simulated 25 times, and the results of variable selection using SurvNet are averaged over these 25 simulations. For Dataset 1, we demonstrate how SurvNet works step by step to look into its behavior, and we also study the influence of the elimination rate by setting $\varepsilon$ to different values. On other simulation datasets, results are similar and thus are not given.

\subsection*{Dataset 1: simulated data with independent variables}
We simulated a $10,000\times784$ matrix $\bm{X}$, with $x_{ij} \sim {\rm i.i.d.}\ U(0,1)$ for $1 \le i \le 10,000$, $1 \le j \le 784$, where $U$ means uniform distribution, and treated its rows and columns as samples and variables respectively. The samples were randomly assigned into two classes $C_1$ and $C_2$ of equal size. Then $p^\prime = 64$ variables were chosen at random and their values in one class were shifted: for each of these variables, we generated a shift value $\delta_j \sim U(0.1,0.3)$, with its direction having equal probability of being positive and negative. More precisely, $x_{ij} \leftarrow x_{ij} + (2\alpha_j-1) \cdot \delta_j$ for $i \in C_1$, $j \in \Omega_{p^\prime}$, where $\alpha_j \sim {\rm Bernoulli}(\frac{1}{2})$ and $\Omega_{p^\prime}$ was the set of $p^\prime$ randomly chosen variables. In this way, the 784 variables were independent from each other, and the 64 variables were significant because each of them had different mean values in the two classes. This ``independent-variable differential-mean'' scheme is a very widely used simulation scheme for studying variable selection.

We ran SurvNet on this dataset with an FDR cutoff $\eta^* = 0.1$ and an elimination rate $\varepsilon = 1$. To demonstrate how SurvNet works step by step, Figure \ref{fig:2}a shows, in one instance of simulation, the number of original variables and surrogate variables left at each step of a selection process as well as the corresponding estimated FDR. The number of variables to be eliminated in the subsequent step is also displayed, and notice that our algorithm was efficient: it eliminated a large number of variables at the beginning and gradually slowed down the elimination as the number of remaining variables decreased and the estimated FDR got closer to the desired value. When the estimated FDR became less than 0.1, the selection process stopped, and the final model turned out to contain all the 64 truly significant variables. On the same data, we studied the influence of elimination rate, and the results of using $\varepsilon = 1$ and $\varepsilon = 0.5$ are shown in Figure \ref{fig:2}b and \ref{fig:2}c. It is found that while a larger elimination rate led to a faster selection process with fewer steps, the number of variables left at the end of the selection was almost the same (Figure \ref{fig:2}b). Moreover, regardless of elimination rate, our method gave an accurate estimate of FDR, and the true value of FDR was well controlled throughout the selection process (Figure \ref{fig:2}c).

The overall performance of SurvNet under $\eta^* = 0.1$ and $\varepsilon = 1$ was summarized in Table \ref{tab:1}. The test loss and test error on the model with selected variables were both less than those on the model that contains all original variables, indicating enhanced predictive power of the network. More importantly, SurvNet accurately selected the significant variables: it kept 61.92 of the 64 significant variables, along with 7.42 false positives, and the selected variables had an FDR of 0.105, which was very close to the cutoff value 0.1. The estimated FDR, 0.093, was also close to the actual FDR.

The results under different elimination rates ($\varepsilon = 1$ and $\varepsilon = 0.5$), different FDR cutoffs ($\eta^* = 0.1$ and $\eta^* = 0.05$), and different numbers of significant variables ($p^\prime = 64$ and $p^\prime = 32$) are shown in Table S1.

\subsection*{Dataset 2: simulated data with correlated variables}
We considered correlated variables in this simulation dataset. It is well known that variable dependence often makes FDR estimation difficult \cite{benjamini2001control, heesen2015inequalities}, and we wondered whether SurvNet was still valid in this case. Images are perfect examples of data with correlated variables, as the value of a pixel usually highly depends on the value of its surrounding pixels. Here we used all images of digit 0 in the MNIST data and randomly assigned them into two classes, and all variables were supposed to be non-significant for classification at this time. Then we picked $p^\prime = 64$ variables and shifted their mean values in one class in the same way we did in Dataset 1.

Table \ref{tab:1} shows the performance of SurvNet under $\eta^* = 0.1$ and $\varepsilon = 1$. Similar to that on Dataset 1, the test loss decreased after variable selection. The test error before and after variable selection were both zero, possibly due to the positive correlation between pixels, which reduced the difficulty of the classification problem. Although SurvNet identified slightly fewer significant variables (59.36 of the 64 significant variables) than it did in Dataset 1, the FDR 0.107 was still very close to the desired cutoff, and its estimated value 0.094 was accurate as well. For results under different sets of parameter values, see Table S2.

\subsection*{Dataset 3: simulated data with variance-inflated variables}
The third simulation scheme is very challenging. Unlike in the previous two datasets, the significant variables did not differ in the mean values of the two classes; instead, they differed only in the variances. Same as in Dataset 1, we simulated a $10,000\times784$ matrix $\bm{X}$ whose element $x_{ij} \sim {\rm i.i.d.}\ U(0,1)$ and divided the samples into two equal-size classes $C_1$ and $C_2$. But then, to make $p^\prime=64$ randomly chosen variables significant, we let $x_{ij} \leftarrow x_{ij} + (2\alpha_{ij}-1) \cdot \delta_{ij}$ for $i \in C_1$, $j \in \Omega_{p^\prime}$, where $\alpha_{ij} \sim {\rm Bernoulli}(\frac{1}{2})$, and $\delta_{ij} \sim U(0.8,1)$. Note that different from the first two simulation schemes, here $\delta$ and $\alpha$ depend on both $i$ and $j$. Thus the means of these variables remained unchanged, but their standard deviations inflated from 0.29 to 0.95 (see Supplementary Materials for calculations). In other words, the only difference between the two classes was that the values of 64 out of 784 pixels were ``noisier''. In this case, classifiers and tests based on discrepancies in the mean values would fail. For example, t-test merely identified 0.20 (averaged over 25 instances) of the 64 significant variables.

The results of applying SurvNet with $\eta^* = 0.1$ and $\varepsilon = 1$ were shown in Table \ref{tab:1}, and the first thing to notice is the dramatic improvement of classification accuracy on the test set. While the test error given by the network with all 784 vriables was 49.42\%, it dropped to 0.47\% after variable selection by SurvNet; that is, from an almost random guess to an almost perfect classification. This implies that the variable selection gives back to the DNN the ability of utilizing all types of information useful for classification, which was masked by the overwhelming irrelevant variables. 
Among the selected variables, 23.00 were truly significant variables, and 3.40 were false positives. Although only 36\% of the significant variables were successfully identified, the FDR of the remaining variables, 0.114, was close to the cutoff, and the estimated FDR was acceptably accurate.

We then scrutinized the selection process of SurvNet on this dataset, and found that the reason only a proportion of significant variables were retained was that the initial network that made almost random guesses could not accurately determine the importance of variables and thus many significant variables were removed. As the selection proceeded, the network gained higher classification accuracy and also stronger ability to distinguish the significant variables. When we used a smaller elimination rate, say $\varepsilon = 0.5$, SurvNet was able to keep a larger proportion of significant variables (see Table S3 for details).

\subsection*{Dataset 4: simulated regression data}
Suppose the data matrix is $\bm{X}=(x_{ij})_{10,000 \times 784}$, and each $x_{ij} \sim U(-1,1)$. Of the 784 variables, 64 were randomly chosen as significant variables (denoted by $x_{k_j}, j=1,\ldots,64$), and $y$ was set to be the linear combination of $x_{k_j}$ or its nonlinear functions, plus a few interaction terms and a random error term:
\[
\begin{split}
y_i=\sum_{j=1}^{16} \beta_j x_{ik_j} + \sum_{j=17}^{32} \beta_j \sin x_{ik_j} + \sum_{j=33}^{48} \beta_j e^{x_{ik_j}} + \sum_{j=49}^{64} \beta_j \max (0,x_{ik_j}) \\
+ \beta_1^\prime x_{ik_{15}} x_{ik_{16}} + \beta_2^\prime x_{ik_{31}} x_{ik_{32}} + \beta_3^\prime x_{ik_{47}} x_{ik_{48}} + \beta_4^\prime x_{ik_{63}} x_{ik_{64}} + \varepsilon_i,
\end{split}
\]
where $\beta_j=(2\alpha_j-1) \cdot b_j$, $\alpha_j \sim {\rm Bernoulli}(\frac{1}{2})$, $b_j \sim U(1,3)$, $\varepsilon_i \sim N(0,1)$ for $i=1,\ldots,10,000$, $j=1,\ldots,64$, and $\beta_1^\prime$, $\beta_2^\prime$, $\beta_3^\prime$, $\beta_4^\prime$ have the same distribution as $\beta_j$.

We ran SurvNet with $\eta^*=0.1$ and $\varepsilon=1$ on 25 instances of simulation and the results are reported in the format of mean $\pm$ standard deviation. After variable selection, the test loss was reduced greatly, from $33.013\pm27.059$ to $8.901\pm1.988$. The number of remaining original variables was $71.16\pm5.02$ on average, and $63.96\pm0.20$ of the 64 significant variables were kept. The actual FDR of the selected variables was $0.097\pm0.061$, close to the desired value 0.1, and the estimated FDR, $0.094\pm0.004$, was accurate. The results suggest that SurvNet is highly effective for this regression dataset.

\subsection*{Dataset 5: digits 4 and 9 in MNIST}
After four simulation datasets, we applied SurvNet to the MNIST data. Here we only used the images of two digits that look alike (4 and 9), as they are similar in most pixels and are only different in pixels in certain regions. In Figure \ref{fig:3}a, we show two representative 4's that differ in the width of top opening and two representative 9's that differ in the presence of a bottom hook. The four regions circled in red are likely to be most significant in differentiating 4's and 9's, especially the region in the upper middle denoting whether the top is closed or open, and the region in the lower middle denoting whether there is a hook at the bottom.

From left to right, Figure \ref{fig:3}b shows the pixels that were selected by SurvNet under four combinations of FDR cutoffs ($\eta^* = 0.1$ or 0.01) and elimination rates ($\varepsilon = 1$ or 0.5). The colors display the relative importance, defined by equation \ref{eq:Sj_L2} (see Methods), of the selected pixels, and a darker color means greater importance. We found that different parameter settings gave quite consistent results, and they all picked out the four regions that were speculated to be significant.

\subsection*{Dataset 6: single-cell RNA-Seq data}
Chen \textit{et al.} performed single-cell RNA-Seq analysis of the adult mouse hypothalamus and identified 45 cell types based on clustering analysis \cite{chen_single-cell_2017}. We used 5,282 cells in two non-neuronal clusters, oligodendrocyte precursor cell (OPC) and myelinating oligodendrocyte (MO), which reflected two distinct stages of oligodendrocyte maturation. Following a standard pre-processing protocol of single-cell RNA-Seq data \cite{hwang2018single}, we filtered out the genes whose expression could not be detected in more than 30\% of these cells, which left 1,046 genes for further analysis, and used $\log({\rm TPM}+1)$ for measuring gene expression levels, where TPM standed for ``transcripts per million''.

With $\eta^*=0.01$ and $\varepsilon=1$, SurvNet selected 145 genes in one realization. Figure \ref{fig:4} shows the heatmap of the expression values of these genes, in which rows are genes and columns are cells. The top banner shows the class labels for the samples. For gene expression data, the set of significant genes are typically identified by ``differential expression'' analysis, which finds differences in the mean expression levels of genes between classes. Indeed, as the heatmap shows, most genes have evidently different mean expression levels in the OPCs and MOs. However, among the 145 significant genes identified by SurvNet, 16 have log-fold-changes (logFCs) less than 1, meaning that their mean expression levels are not very different in the OPCs and MOs. In Figure \ref{fig:4}, these genes are marked in purple on the left banner, in contrast to green for the other genes. Actually, Bartlett's test, which tests the difference in variance, claimed that 14 of these 16 genes had unequal variances in the two groups of cells (p-value < 0.05); thus, they are instances of variance-inflated variables selected by SurvNet, in addition to the ones in Dataset 3. Again, SurvNet demonstrates its ability to identify various types of significant variables, not just variables with different means.

Further, the functional interpretations of the selected genes match the biological characteristics of OPCs and MOs. We conducted Gene Ontology (GO) analysis using DAVID 6.8 program \cite{huang2009systematic,da2009bioinformatics}, and found that these genes were likely to play an important role in a number of biological processes, for example, substantia nigra development (with p-value $8.8\times10^{-9}$, fold enrichment 29.1), nervous system development ($1.8\times10^{-5}$, 4.8), positive regulation of dendritic spine development ($1.2\times10^{-3}$, 19.0) and astrocyte differentiation ($3.2\times10^{-3}$, 34.5). In particular, oligodendrocyte differentiation ($1.8\times10^{-3}$, 16.2) defines the transition from OPCs to their mature form (MOs) \cite{rubio2004vitro,barateiro2014temporal}, and myelination ($7.1\times10^{-5}$, 13.8), which is the process of generating myelin and is a kind of axon ensheathment ($3.5\times10^{-2}$, 55.2), is unique to MOs \cite{menn2006origin,barateiro2014temporal}. Corresponding to these processes, the selected genes were also enriched for cellular components such as myelin sheath ($2.4\times10^{-19}$, 16.2), axon ($1.2\times10^{-5}$, 5.0) as well as internode region of axon ($2.9\times10^{-4}$, 106.1), and molecular functions like structural constituent of myelin sheath ($3.1\times10^{-6}$, 115.3). Besides, among the 16 selected genes whose expression levels had no obvious differences in the OPCs and MOs, \textit{Cd9} was involved in oligodendrocyte development \cite{terada2002tetraspanin}, and \textit{Ckb}, \textit{Actb}, \textit{Tuba1a} as well as \textit{Gpm6b} were related to myelin sheath or myelin proteolipid protein PLP \cite{jahn2009myelin,werner2013critical}.

After variable selection, the test loss was reduced from $4.230\times10^{-3}$ to $3.460\times10^{-3}$, and the test error dropped from 0.083\% to 0.076\% (averaged over 25 realizations).

\section*{Conclusions and discussion}
We have presented a largely automatic procedure for variable selection in neural networks (SurvNet). It is based on a new measure of variable importance that applies to a variety of networks, deep or shallow, for regression or classification, and with one or multiple output units. More importantly, SurvNet is the first method that estimates and controls the FDR of selected variables, which is essential for applications where the trustworthiness of variable selection is pivotal. By introducing surrogate variables, it avoids training multiple networks in parallel. SurvNet also adjusts the number of variables to eliminate at each step, and the ``warm start'' nature of backward elimination facilitates the training of networks. On multiple simulation datasets and real datasets, SurvNet has effectively identified the significant variables and given a dependable estimate of FDR.

SurvNet takes advantages of modern developments of DNNs. The importance scores of input variables that are based on derivatives with respect to the inputs can be efficiently computed by functions in deep-learning packages such as TensorFlow, PyTorch, and Theano. Moreover, advances in optimization techniques and computation platforms have made the training of DNNs highly scalable. In particular, DNNs can accommodate a large number of input variables, which enables the introduction of surrogate variables.

Given a dataset, SurvNet may select different sets of significant variables at different runs owing to the randomness originated from the generation of surrogate variables and the training of networks (e.g., the random initial values of weights). While the former is unique to SurvNet, the latter is ubiquitous to any applications of neural networks. The randomness caused by generating surrogate variables may be lowered by, for example, using a larger number of surrogate variables or assembling results from multiple runs, but this randomness should not be a major concern if it is not much larger than the inevitable randomness coming from network training. To study this, we take Dataset 5 as an example. Using $\eta^*=0.1$ and $\varepsilon=1$, we ran SurvNet 25 times, and found that SurvNet selected $114.16\pm11.36$ variables; and the overlapped proportion of the selected variables in each pair of realizations was approximately 0.77. These results reflected both sources of randomness. Then we fixed the surrogate variables in each realization, and SurvNet selected $118.32\pm7.94$ variables, with the overlapped proportion of the selected variables of each pair of realizations around 0.79. This indicates that for this dataset, the randomness brought by surrogate variables was much less than that by the training of networks. And (hopefully) as peace of mind, some other well-known techniques for statistical tests and variable selection, such as permutation tests and bootstrap tests (and especially, parametric bootstrap tests), also have extra randomness caused by permutations or random number generations, but they are still very widely used.

Next we discuss how many surrogate variables should be generated. In all experiments in this paper, we simply set the number of surrogate variables ($q$) to be the same as the number of original variables ($p$). A larger $q$ may lower the randomness brought by the surrogate variables and thus give a more stable selection of variables and a more accurate estimate of FDR. These improvements can be noticeable and worth pursuing when the number of original variables is small. On the other hand, a larger number of surrogate variables may increase the computational load. As a rule of thumb, we recommend using $q=p$ for datasets with moderate to large sample size, and $q$ can be a few times larger than $p$ if $p$ is small and be smaller than $p$ if $p$ is very large.

Although variable selection is critical to many real applications and is often considered one of the most fundamental problems in machine learning \cite{trevor2009elements, tibshirani2015statistical}, it is worth noting that this task does not apply to certain problems or certain types of DNNs. As an example, for some image datasets like ImageNet \cite{deng2009imagenet}, deep convolutional neural networks are often a good choice due to their translation invariance characteristics, as the object of interest, such as a dog, may appear at any position in an image and thus theoretically every pixel should be relevant. Also, in the area of natural language processing, where recurrent neural networks are often used, the number of input variables (i.e. the length of input sequence) is not fixed and variable selection makes little sense.

The main aim of variable selection is to identify significant variables, which may, for example, shed light on the mechanisms of biological processes or guide further experimental validation. Apart from that, an additional aim may be to improve the classification accuracy. Although we did observe an improvement of generalization accuracy on all our simulated and real datasets, such an improvement is not guaranteed even if the variable selection procedure works perfectly. In some datasets, except for a set of significant variables, all other variables are almost completely irrelevant to the outcome, and variable selection may give extra power in prediction. However, in some other datasets, the relevances of variables are not polarized; there are many variables each having a very small influence on the output, but their accumulative contribution is non-negligible. For these datasets, such variables are likely to be ruled out during selection since it is hard to confidently determine their individual significance, but ignoring all of them could cause a loss of prediction power.

\section*{Methods}

\subsection*{Measures of variable importance}

\subsubsection*{Notation}
We use a tuple ($\bm{x}$,$\bm{y}$) to represent the input and the output of the network, with $\bm{y}$ being either one-dimensional or multi-dimensional. $x_j$ denotes the $j^\mathrm{th}$ component of $\bm{x}$, namely the $j^\mathrm{th}$ variable, and ($\bm{x}^{(i)}$,$\bm{y}^{(i)}$) ($i=1,\ldots,n$) is the $i^\mathrm{th}$ sample, where $n$ is the total number of samples (in the training set). Given a proper form of the loss $L(\cdot,\cdot)$, the loss function $L^*=\sum_{i=1}^n L(\bm{y}^{(i)},f(\bm{x}^{(i)}))$, where $f$ denotes the output function of the network. The most popular choices for $L(\cdot,\cdot)$ are the squared error loss for regression problems and the cross-entropy loss for classification problems.

\subsubsection*{Existing measures}
Many statistics have been proposed to measure the importance of variables in neural networks, and they generally fall into two categories \cite{tetko_neural_1996, steppe_feature_1997}.

One category of methods estimate the importance of $x_j$, denoted by $S_j$, based on the magnitudes of the connection weights in the network \cite{sen_predicting_1995, yacoub_hvs:_1997, garson_interpreting_1991, nath_determining_1997, gevrey_review_2003}. A simple example is the sum of absolute values of input weights \cite{sen_predicting_1995}, but larger values of weights in the input layer do not mean greater importance if connections in hidden layers have small weights, and a better alternative is to replace the input weights with the products of the weights on each path from this input to the output \cite{yacoub_hvs:_1997}. These measures were developed for networks with only one hidden layer, and they are unlikely to work well for deeper networks as the outgoing weights of a neuron does not reflect its importance once the neuron is inactive (e.g., when the input of a sigmoid neuron is far from zero or the input of a ReLU neuron is negative).

The other category of methods estimate $S_j$ by the sum of influences of the input weights on the loss function, i.e. $S_j=\sum_{k \in \Omega_j} \delta L^*_k$, where $\Omega_j$ is the set of outgoing weights from the $j^\mathrm{th}$ input neuron, and $\delta L^*_k$ is the increment of the loss function caused by the removal of weight $w_k$ \cite{tetko_neural_1996}. $\delta L^*_k$ can be approximated by a Taylor series of the loss function using first-order terms \cite{mozer_skeletonization:_1989, karnin_simple_1990} or second-order terms \cite{lecun_optimal_1990, cibas_variable_1994, hassibi_second_1993}. However, it is unclear why $S_j$ equals the (unweighted) sum of $\delta L^*_k$'s.

Apart from these two major categories of measures, it was also proposed to use $S_j=\frac{\partial f}{\partial x_j}$, i.e. $S_j=\frac{\partial y}{\partial x_j}$, when the output $y$ is one-dimensional \cite{dimopoulos_use_1995,dimopoulos_neural_1999}. But it is unclear how $S_j$ should be defined when there are multiple output units. Let $y_1,\ldots,y_K$ be the output values of $K$ output units, and one definition of $S_j$ was given by $S_j=\sum_{k=1}^{K} |\frac{\partial y_k}{\partial x_j}|$ \cite{ruck_feature_1990}. However, using this summation seems problematic in some cases, especially when $y_1,\ldots,y_K$ are the outputs of softmax functions.

\subsubsection*{Our new measure}
We propose a simple and direct measure of the importance of variable $j$ based on $\frac{\partial L}{\partial x_j}$, which describes how the loss changes with $x_j$. There are a few advantages of using $\frac{\partial L}{\partial x_j}$. First, regardless of the structure of the network and whether the output(s) is/are continuous or categorical, $L$ is always well defined since it is the target for the optimization/training of the network. Thus the proposed measure is applicable to a wide variety of networks. Second, no matter how many output units there are, $L$ is always a scalar and hence $\frac{\partial L}{\partial x_j}$ is always a scalar. There is no trouble in how to combine effects from multiple output units. Third, $\frac{\partial L}{\partial x_j}$ is easily computable with the backpropogation method, and popular frameworks/libraries for DNN computations (e.g., TensorFlow, PyTorch and Theano) all use differentiators that efficiently compute partial derivatives (gradients) of arbitrary forms.

Note that $\frac{\partial L}{\partial x_j}$ is a function of the tuple ($\bm{x}$,$\bm{y}$), and hence it is natural to estimate it by its mean over all observations in the training set. To avoid cancellation of positive and negative values, we measure the importance of $x_j$ by the mean of absolute values
\begin{equation}
	S_j=\frac{1}{n} \sum_{i=1}^n |\frac{\partial L}{\partial x_j}(\bm{y}^{(i)},f(\bm{x}^{(i)}))|,
	\label{eq:Sj_L1}
\end{equation}
or the mean of squares
\begin{equation}
	S_j=\frac{1}{n} \sum_{i=1}^n \frac{\partial L}{\partial x_j}(\bm{y}^{(i)},f(\bm{x}^{(i)}))^2,
	\label{eq:Sj_L2}
\end{equation}
where $\frac{\partial L}{\partial x_j}(\bm{y}^{(i)},f(\bm{x}^{(i)}))$ is the value of $\frac{\partial L}{\partial x_j}$ at the $i$'th training sample.

The importance scores given by equation \ref{eq:Sj_L1} and equation \ref{eq:Sj_L2} implicitly assume that all the input values have similar range, which is typically the case for DNNs, since it is common practice to standardize/scale the variables before supplying them to the network for the sake of faster and more stable training of the network \cite{bishop1995neural, lecun2012efficient}. If this is not the case, we suggest the score in equation \ref{eq:Sj_L1} be multiplied by the (sample) standard deviation of $x_j$ and the score in equation \ref{eq:Sj_L2} be multiplied by the (sample) variance of $x_j$.

Note that in the case of multiple linear regression, $L = \frac{1}{2} (y-\hat{y})^2 = \frac{1}{2} (y-\sum_j \beta_j x_j)^2$, where $y$ is a scalar response and $\beta_j$ is the $j^\mathrm{th}$ regression coefficient, then $\frac{\partial L}{\partial x_j}=-(y-\hat{y})\beta_j$. Thus, $S_j$ is defined as $|\beta_j| \cdot \frac{1}{n} \sum_{i=1}^n |e_i|$ or $\beta_j^2 \cdot \frac{1}{n} \sum_{i=1}^n e_i^2 $ by (1) and (2) respectively, where $e_i=y^{(i)}-\hat{y}^{(i)}$. Note that $S_j$ is proportional to $|\beta_j|$ or $\beta_j^2$ as $\frac{1}{n} \sum_{i=1}^n |e_i|$ and $\frac{1}{n} \sum_{i=1}^n e_i^2$ are constants. Therefore, both of them are reasonable measures of the contribution of the $j^\mathrm{th}$ variable, and they are actually equivalent in this case. The meaning of $S_j$ in some other special cases, such as linear regression with multiple outputs and logistic regression with one or multiple outputs, is elaborated in Supplementary Materials.

All results in the main text were obtained using equation \ref{eq:Sj_L2}. Results obtained using equation \ref{eq:Sj_L1} (given in Supplementary Materials) are not significantly different.

\subsection*{Elimination procedure with FDR control}
In this section, we first introduce how we estimate FDR and then talk about how we use this estimate to determine the number of variables to eliminate at each step.

\subsubsection*{Introduction of surrogate variables}
The key of estimating FDR \cite{storey2003statistical} is to estimate/generate the null distribution of the test statistic. In our case, it is to obtain the distribution of the importance score $S_j$ defined by equation \ref{eq:Sj_L2} or equation \ref{eq:Sj_L1} for variables that are not significant. Since the network is a complicated and highly nonlinear model, a theoretical distribution that applies to various network structure and various types of data may not exist. This null distribution needs to be obtained for the network and the data in hand.

However, it is usually unknown which variables are truly null. If we construct the null distribution by permuting the output values of the data, it seems inevitable to train multiple networks from scratch in parallel. For this reason, we propose to introduce/add a number of variables that are known/generated to be null. We call these variables ``surrogate null variables'' (or ``surrogate variables'' for short). These variables will be concatenated with the original variables to form a larger data matrix.

To be precise, suppose there are $p$ original variables and $n$ training samples (including validation samples). Then after we add $q$ surrogate variables, the new data matrix will be of size $n\times(p+q)$, which binds the original $n\times p$ data matrix $\bm{X}$ with a $n\times q$ data matrix for surrogate variables $\bm{X}_s$. It is assumed that the original variables are distributed in similar ranges or have been standardized, which is a suggested pre-processing step as it benefits the training of the network, and the elements in $\bm{X}_s$ are sampled with replacement (or without replacement when $q \le p$) from the elements in $\bm{X}$. As a result, the $q$ surrogate variables are null, and their importance scores give the null distribution.

We recommend $q$ to be on the same scale as $p$ (see Conclusions and Discussion for a more detailed discussion about the choice of $q$). For convenience, $q$ takes the same value as $p$ in all experiments in this paper. In this case, the elements in $\bm{X}_s$ can be generated by permuting the elements in $\bm{X}$.

The selection procedure of SurvNet starts with using all $p+q$ variables as inputs. Then at each step, it eliminates a number of least important variables, including both original variables and surrogate variables. The remaining variables are used to continue training the network, and the elimination stops once the FDR falls below the cutoff.

\subsubsection*{FDR estimation}
Then we consider how to estimate FDR at any given time of the selection process. Suppose $r$ variables are retained in the network, among which there are $r_0$ surrogate variables, then $r_0/q$ proportion of surrogate (null) variables have not been eliminated yet. Accordingly, one would expect that roughly the same proportion of null original variables still exist at this time, that is, approximately $\frac{r_0}{q} \cdot p_0$ variables among the remaining original variables are falsely called significant, where $p_0$ is the number of null variables in the original dataset. Thus, an estimate of the FDR of the $r-r_0$ original variables is given by
\begin{equation}
	\tilde{\eta}=\frac{\frac{r_0}{q} \cdot p_0}{r-r_0}
\end{equation}
In practice, however, $p_0$ is unknown, and a common strategy is to replace it with its upper bound $p$ \cite{storey2003statistical}. Hence we have the following estimated FDR,
\begin{equation}
	\hat{\eta}=\frac{\frac{r_0}{q} \cdot p}{r-r_0}=\frac{r_0}{r-r_0} \cdot \frac{p}{q}
\label{eq:est_fdr}
\end{equation}
Apparently, when $\hat{\eta}$ is controlled to be no greater than a pre-specified threshold $\eta^*$, $\tilde{\eta}$ is guaranteed to be no greater than $\eta^*$ as well. When $q=p$, $\hat{\eta}$ can be simplified as $\frac{r_0}{r-r_0}$.

\subsubsection*{Determination of the number of variables to eliminate}
If the estimated FDR $\hat{\eta}$ (given by equation \ref{eq:est_fdr}) is less than or equal to the FDR cutoff $\eta^*$, the variable selection procedure stops. Otherwise, the procedure proceeds, and we want to decide how many variables to eliminate among the $r$ variables that are still in the model. Let this number be $m$, and the determination of $m$ is based on the following considerations. On one hand, we expect that the elimination process is time-saving and reaches the FDR threshold quickly; on the other hand, we want to avoid eliminating too many variables at each step, in which case the FDR may fall much lower than the threshold. We have
\begin{thm}
	If $m$ variables are further eliminated from the current model, the smallest possible estimated FDR after this step of elimination is
	\begin{equation}
	\min \hat{\eta}^{\rm new}= (1-\frac{m}{r_0})\cdot \hat{\eta},
	\label{eq:rm_new}
	\end{equation}
	where $r_0$ is the number of surrogate variables that are in the model before this step of elimination.
\end{thm}
\begin{proof}
	Suppose there are $m_0$ surrogate variables among the $m$ variables to be eliminated, $0 \le m_0 \le m$, then according to equation \ref{eq:est_fdr}, $\hat{\eta}$ will be updated to
	\begin{equation}
	\hat{\eta}^{\rm new}=\frac{r_0-m_0}{r-r_0-(m-m_0)} \cdot \frac{p}{q}.
	\end{equation}
	Note that $\hat{\eta}^{\rm new}$ is monotonically decreasing with respect to $m_0$ for any fixed  $m$, we have
	\begin{equation}
	\min \hat{\eta}^{\rm new}=\hat{\eta}^{\rm new}|_{m_0=m}=\frac{r_0-m}{r-r_0} \cdot \frac{p}{q}.
	\label{eq:rm_new_int}
	\end{equation}
	Equation \ref{eq:est_fdr} indicates that $\frac{1}{r - r_0} \cdot \frac{p}{q}=\frac{\hat{\eta}}{r_0}$. Plugging it into \ref{eq:rm_new_int}, we have
	\[
	\min \hat{\eta}^{\rm new}=(r_0-m) \cdot \frac{\hat{\eta}}{r_0} = (1-\frac{m}{r_0})\cdot \hat{\eta}.
	\]
\end{proof}

It follows from equation \ref{eq:rm_new} that $\min \hat{\eta}^{\rm new} = \eta^*$ when $m=(1-\frac{\eta^*}{\hat{\eta}}) \cdot r_0$. Also, note that $\min \hat{\eta}^{\rm new}$ is a monotonically decreasing function of $m$. Therefore, when $m < (1-\frac{\eta^*}{\hat{\eta}}) \cdot r_0$, $\min \hat{\eta}^{\rm new} > \eta^*$ and thus $\hat{\eta}^{\rm new} > \eta^*$. That is,
\begin{cor}
	When $m < (1-\frac{\eta^*}{\hat{\eta}}) \cdot r_0$, the estimated FDR after this step of elimination $\hat{\eta}^{\rm new}$ is guaranteed to be still greater than the FDR cutoff $\eta ^*$.
\end{cor}
On the other hand, when $m \ge (1-\frac{\eta^*}{\hat{\eta}}) \cdot r_0$, $\min \hat{\eta}^{\rm new} \le \eta^*$. That is,
\begin{cor}
	When $m \ge (1-\frac{\eta^*}{\hat{\eta}}) \cdot r_0$, the estimated FDR after this step of elimination $\hat{\eta}^{\rm new}$ may reach the FDR cutoff $\eta ^*$.
\end{cor}

Corollary 1 says that $m$ values less than $(1-\frac{\eta^*}{\hat{\eta}}) \cdot r_0$ are ``safe'' but the elimination will not stop after this step. Corollary 2 says that $m$ values much larger than $(1-\frac{\eta^*}{\hat{\eta}}) \cdot r_0$ may not be ``safe'' anymore. Taking both into consideration, we choose the step size to be
\begin{equation}
    m=\lceil (1-\frac{\eta^*}{\hat{\eta}}) \cdot r_0 \rceil,
    \label{eq:m}
\end{equation}
where $\lceil \cdot \rceil$ denotes ``ceiling'', i.e. the smallest integer that is no less than $\cdot$. Notice that when $\hat{\eta} > \eta^*$, which is the premise of continuing to eliminate variables, $1-\frac{\eta^*}{\hat{\eta}}>0$, and $r_0>0$ as well since $\hat{\eta}$ is positive. Thus $m$ is ensured to be no less than 1 at each step of variable elimination.

This form of $m$ seems to be quite reasonable for the following reasons. First, if there still remain a great number of surrogate variables in the network, clearly more of them should be taken out. As $r_0$ decreases, $m$ will be smaller, and this makes sense since one should be more careful in further elimination. Second, when $\hat{\eta}$ is much higher than $\eta^*$, one will naturally expect a larger $m$ so that the updated estimated FDR will approach this cutoff.

Using the $m$ determined by equation \ref{eq:m}, there is a chance that the estimated FDR will get to the cutoff in only one step. Many times such a fast pace is not preferred as removing too many inputs at a time may make our warm start of the training not warm any more. Hence we may introduce an ``elimination rate'' $\varepsilon$, which is a constant between 0 and 1, and take
\begin{equation}
    m=\lceil \varepsilon \cdot (1-\frac{\eta^*}{\hat{\eta}}) \cdot r_0 \rceil.
\end{equation}

\section*{Author Contributions}
J.L. conceived the study, J.L. and Z.S. proposed the methods, Z.S. implemented the methods and constructed the data analysis, Z.S. drafted the manuscript, J.L. substantively revised it.

\section*{Competing Interests statement}
The authors declare no competing interests.

\bibliographystyle{unsrt}

\begin{thebibliography}{10}
	
	\bibitem{may_review_2011-1}
	Robert May, Graeme Dandy, and Holger Maier.
	\newblock Review of input variable selection methods for artificial neural
	networks.
	\newblock {\em Artificial neural networks-methodological advances and
		biomedical applications}, 10:16004, 2011.
	
	\bibitem{guyon_introduction_2003}
	Isabelle Guyon and André Elisseeff.
	\newblock An introduction to variable and feature selection.
	\newblock {\em Journal of machine learning research}, 3(Mar):1157--1182, 2003.
	
	\bibitem{chandrashekar_survey_2014}
	Girish Chandrashekar and Ferat Sahin.
	\newblock A survey on feature selection methods.
	\newblock {\em Computers \& Electrical Engineering}, 40(1):16--28, 2014.
	
	\bibitem{battiti_using_1994}
	Roberto Battiti.
	\newblock Using mutual information for selecting features in supervised neural
	net learning.
	\newblock {\em IEEE Transactions on neural networks}, 5(4):537--550, 1994.
	
	\bibitem{may_non-linear_2008}
	Robert~J. May, Holger~R. Maier, Graeme~C. Dandy, and TMK~Gayani Fernando.
	\newblock Non-linear variable selection for artificial neural networks using
	partial mutual information.
	\newblock {\em Environmental Modelling \& Software}, 23(10-11):1312--1326,
	2008.
	
	\bibitem{sung1998ranking}
	AH~Sung.
	\newblock Ranking importance of input parameters of neural networks.
	\newblock {\em Expert Systems with Applications}, 15(3-4):405--411, 1998.
	
	\bibitem{maier1998use}
	Holger~R Maier, Graeme~C Dandy, and Michael~D Burch.
	\newblock Use of artificial neural networks for modelling cyanobacteria
	anabaena spp. in the river murray, south australia.
	\newblock {\em Ecological Modelling}, 105(2-3):257--272, 1998.
	
	\bibitem{brill_fast_1992}
	Frank~Z. Brill, Donald~E. Brown, and Worthy~N. Martin.
	\newblock Fast generic selection of features for neural network classifiers.
	\newblock {\em IEEE Transactions on Neural Networks}, 3(2):324--328, 1992.
	
	\bibitem{tong_genetic_2010}
	Dong~Ling Tong and Robert Mintram.
	\newblock Genetic {Algorithm}-{Neural} {Network} ({GANN}): a study of neural
	network activation functions and depth of genetic algorithm search applied to
	feature selection.
	\newblock {\em International Journal of Machine Learning and Cybernetics},
	1(1):75--87, 2010.
	
	\bibitem{sivagaminathan_hybrid_2007}
	Rahul~Karthik Sivagaminathan and Sreeram Ramakrishnan.
	\newblock A hybrid approach for feature subset selection using neural networks
	and ant colony optimization.
	\newblock {\em Expert systems with applications}, 33(1):49--60, 2007.
	
	\bibitem{grandvalet_outcomes_1999}
	Yves Grandvalet and Stéphane Canu.
	\newblock Outcomes of the equivalence of adaptive ridge with least absolute
	shrinkage.
	\newblock In {\em Advances in neural information processing systems}, pages
	445--451, 1999.
	
	\bibitem{chapados_input_2001}
	Nicolas Chapados and Yoshua Bengio.
	\newblock Input decay: {Simple} and effective soft variable selection.
	\newblock In {\em {IJCNN}'01. {International} {Joint} {Conference} on {Neural}
		{Networks}. {Proceedings} ({Cat}. {No}. 01CH37222)}, volume~2, pages
	1233--1237. IEEE, 2001.
	
	\bibitem{simila_combined_2009}
	Timo Similä and Jarkko Tikka.
	\newblock Combined input variable selection and model complexity control for
	nonlinear regression.
	\newblock {\em Pattern Recognition Letters}, 30(3):231--236, 2009.
	
	\bibitem{scardapane_group_2017}
	Simone Scardapane, Danilo Comminiello, Amir Hussain, and Aurelio Uncini.
	\newblock Group sparse regularization for deep neural networks.
	\newblock {\em Neurocomputing}, 241:81--89, 2017.
	
	\bibitem{zhang_neural_2000}
	Guoqiang~Peter Zhang.
	\newblock Neural networks for classification: a survey.
	\newblock {\em IEEE Transactions on Systems, Man, and Cybernetics, Part C
		(Applications and Reviews)}, 30(4):451--462, 2000.
	
	\bibitem{benjamini1995controlling}
	Yoav Benjamini and Yosef Hochberg.
	\newblock Controlling the false discovery rate: a practical and powerful
	approach to multiple testing.
	\newblock {\em Journal of the Royal statistical society: series B
		(Methodological)}, 57(1):289--300, 1995.
	
	\bibitem{storey2003statistical}
	John~D Storey and Robert Tibshirani.
	\newblock Statistical significance for genomewide studies.
	\newblock {\em Proceedings of the National Academy of Sciences},
	100(16):9440--9445, 2003.
	
	\bibitem{lecun1998gradient}
	Yann LeCun, L{\'e}on Bottou, Yoshua Bengio, Patrick Haffner, et~al.
	\newblock Gradient-based learning applied to document recognition.
	\newblock {\em Proceedings of the IEEE}, 86(11):2278--2324, 1998.
	
	\bibitem{kolodziejczyk2015technology}
	Aleksandra~A Kolodziejczyk, Jong~Kyoung Kim, Valentine Svensson, John~C
	Marioni, and Sarah~A Teichmann.
	\newblock The technology and biology of single-cell rna sequencing.
	\newblock {\em Molecular cell}, 58(4):610--620, 2015.
	
	\bibitem{tensorflow2015-whitepaper}
	Mart\'{\i}n Abadi, Ashish Agarwal, Paul Barham, Eugene Brevdo, Zhifeng Chen,
	Craig Citro, Greg~S. Corrado, Andy Davis, Jeffrey Dean, Matthieu Devin,
	Sanjay Ghemawat, Ian Goodfellow, Andrew Harp, Geoffrey Irving, Michael Isard,
	Yangqing Jia, Rafal Jozefowicz, Lukasz Kaiser, Manjunath Kudlur, Josh
	Levenberg, Dan Man\'{e}, Rajat Monga, Sherry Moore, Derek Murray, Chris Olah,
	Mike Schuster, Jonathon Shlens, Benoit Steiner, Ilya Sutskever, Kunal Talwar,
	Paul Tucker, Vincent Vanhoucke, Vijay Vasudevan, Fernanda Vi\'{e}gas, Oriol
	Vinyals, Pete Warden, Martin Wattenberg, Martin Wicke, Yuan Yu, and Xiaoqiang
	Zheng.
	\newblock {TensorFlow}: Large-scale machine learning on heterogeneous systems,
	2015.
	\newblock Software available from tensorflow.org.
	
	\bibitem{benjamini2001control}
	Yoav Benjamini, Daniel Yekutieli, et~al.
	\newblock The control of the false discovery rate in multiple testing under
	dependency.
	\newblock {\em The annals of statistics}, 29(4):1165--1188, 2001.
	
	\bibitem{heesen2015inequalities}
	Philipp Heesen, Arnold Janssen, et~al.
	\newblock Inequalities for the false discovery rate (fdr) under dependence.
	\newblock {\em Electronic Journal of Statistics}, 9(1):679--716, 2015.
	
	\bibitem{chen_single-cell_2017}
	Renchao Chen, Xiaoji Wu, Lan Jiang, and Yi~Zhang.
	\newblock Single-cell {RNA}-seq reveals hypothalamic cell diversity.
	\newblock {\em Cell reports}, 18(13):3227--3241, 2017.
	
	\bibitem{hwang2018single}
	Byungjin Hwang, Ji~Hyun Lee, and Duhee Bang.
	\newblock Single-cell rna sequencing technologies and bioinformatics pipelines.
	\newblock {\em Experimental \& molecular medicine}, 50(8):96, 2018.
	
	\bibitem{huang2009systematic}
	Da~Wei Huang, Brad~T Sherman, and Richard~A Lempicki.
	\newblock Systematic and integrative analysis of large gene lists using david
	bioinformatics resources.
	\newblock {\em Nature protocols}, 4(1):44, 2009.
	
	\bibitem{da2009bioinformatics}
	Brad T~Sherman Da~Wei~Huang and Richard~A Lempicki.
	\newblock Bioinformatics enrichment tools: paths toward the comprehensive
	functional analysis of large gene lists.
	\newblock {\em Nucleic acids research}, 37(1):1, 2009.
	
	\bibitem{rubio2004vitro}
	Nazario Rubio, Rodrigo Rodriguez, and Maria~Angeles Arevalo.
	\newblock In vitro myelination by oligodendrocyte precursor cells transfected
	with the neurotrophin-3 gene.
	\newblock {\em Glia}, 47(1):78--87, 2004.
	
	\bibitem{barateiro2014temporal}
	Andreia Barateiro and Adelaide Fernandes.
	\newblock Temporal oligodendrocyte lineage progression: in vitro models of
	proliferation, differentiation and myelination.
	\newblock {\em Biochimica et Biophysica Acta (BBA)-Molecular Cell Research},
	1843(9):1917--1929, 2014.
	
	\bibitem{menn2006origin}
	B{\'e}n{\'e}dicte Menn, Jose~Manuel Garcia-Verdugo, Cynthia Yaschine, Oscar
	Gonzalez-Perez, David Rowitch, and Arturo Alvarez-Buylla.
	\newblock Origin of oligodendrocytes in the subventricular zone of the adult
	brain.
	\newblock {\em Journal of Neuroscience}, 26(30):7907--7918, 2006.
	
	\bibitem{terada2002tetraspanin}
	Nobuo Terada, Karen Baracskay, Mike Kinter, Shona Melrose, Peter~J Brophy,
	Claude Boucheix, Carl Bjartmar, Grahame Kidd, and Bruce~D Trapp.
	\newblock The tetraspanin protein, cd9, is expressed by progenitor cells
	committed to oligodendrogenesis and is linked to $\beta$1 integrin, cd81, and
	tspan-2.
	\newblock {\em Glia}, 40(3):350--359, 2002.
	
	\bibitem{jahn2009myelin}
	Olaf Jahn, Stefan Tenzer, and Hauke~B Werner.
	\newblock Myelin proteomics: molecular anatomy of an insulating sheath.
	\newblock {\em Molecular neurobiology}, 40(1):55--72, 2009.
	
	\bibitem{werner2013critical}
	Hauke~B Werner, Eva-Maria Kr{\"a}mer-Albers, Nicola Strenzke, Gesine Saher,
	Stefan Tenzer, Yoshiko Ohno-Iwashita, Patricia De~Monasterio-Schrader, Wiebke
	M{\"o}bius, Tobias Moser, Ian~R Griffiths, et~al.
	\newblock A critical role for the cholesterol-associated proteolipids plp and
	m6b in myelination of the central nervous system.
	\newblock {\em Glia}, 61(4):567--586, 2013.
	
	\bibitem{trevor2009elements}
	Hastie Trevor, Tibshirani Robert, and Friedman JH.
	\newblock The elements of statistical learning: data mining, inference, and
	prediction, 2009.
	
	\bibitem{tibshirani2015statistical}
	Robert Tibshirani, Martin Wainwright, and Trevor Hastie.
	\newblock {\em Statistical learning with sparsity: the lasso and
		generalizations}.
	\newblock Chapman and Hall/CRC, 2015.
	
	\bibitem{deng2009imagenet}
	Jia Deng, Wei Dong, Richard Socher, Li-Jia Li, Kai Li, and Li~Fei-Fei.
	\newblock Imagenet: A large-scale hierarchical image database.
	\newblock In {\em 2009 IEEE conference on computer vision and pattern
		recognition}, pages 248--255. Ieee, 2009.
	
	\bibitem{tetko_neural_1996}
	Igor~V. Tetko, Alessandro E.~P. Villa, and David~J. Livingstone.
	\newblock Neural {Network} {Studies}. 2. {Variable} {Selection}.
	\newblock {\em Journal of Chemical Information and Computer Sciences},
	36(4):794--803, 1996.
	
	\bibitem{steppe_feature_1997}
	J.~M. Steppe and K.~W. Bauer~Jr.
	\newblock Feature saliency measures.
	\newblock {\em Computers \& Mathematics with Applications}, 33(8):109--126,
	1997.
	
	\bibitem{sen_predicting_1995}
	Tarun~K. Sen, Robert Oliver, and Nilanjan Sen.
	\newblock Predicting corporate mergers.
	\newblock In {\em Neural networks in the capital markets}, pages 325--340. New
	York: Wiley, 1995.
	
	\bibitem{yacoub_hvs:_1997}
	Meziane Yacoub and Y.~Bennani.
	\newblock {HVS}: {A} heuristic for variable selection in multilayer artificial
	neural network classifier.
	\newblock In {\em Intelligent {Engineering} {Systems} {Through} {Artificial}
		{Neural} {Networks}, {St}. {Louis}, {Missouri}}, volume~7, pages 527--532,
	1997.
	
	\bibitem{garson_interpreting_1991}
	G.~David Garson.
	\newblock Interpreting neural-network connection weights.
	\newblock {\em AI expert}, 6(4):46--51, 1991.
	
	\bibitem{nath_determining_1997}
	Ravinder Nath, Balaji Rajagopalan, and Randy Ryker.
	\newblock Determining the saliency of input variables in neural network
	classifiers.
	\newblock {\em Computers \& Operations Research}, 24(8):767--773, 1997.
	
	\bibitem{gevrey_review_2003}
	Muriel Gevrey, Ioannis Dimopoulos, and Sovan Lek.
	\newblock Review and comparison of methods to study the contribution of
	variables in artificial neural network models.
	\newblock {\em Ecological modelling}, 160(3):249--264, 2003.
	
	\bibitem{mozer_skeletonization:_1989}
	Michael~C. Mozer and Paul Smolensky.
	\newblock Skeletonization: {A} technique for trimming the fat from a network
	via relevance assessment.
	\newblock In {\em Advances in neural information processing systems}, pages
	107--115, 1989.
	
	\bibitem{karnin_simple_1990}
	Ehud~D. Karnin.
	\newblock A simple procedure for pruning back-propagation trained neural
	networks.
	\newblock {\em IEEE transactions on neural networks}, 1(2):239--242, 1990.
	
	\bibitem{lecun_optimal_1990}
	Yann LeCun, John~S. Denker, and Sara~A. Solla.
	\newblock Optimal brain damage.
	\newblock In {\em Advances in neural information processing systems}, pages
	598--605, 1990.
	
	\bibitem{cibas_variable_1994}
	Tautvydas Cibas, Françroise~Fogelman Soulié, Patrick Gallinari, and Sarunas
	Raudys.
	\newblock Variable selection with optimal cell damage.
	\newblock In {\em {ICANN}’94}, pages 727--730. Springer, 1994.
	
	\bibitem{hassibi_second_1993}
	Babak Hassibi and David~G. Stork.
	\newblock Second order derivatives for network pruning: {Optimal} brain
	surgeon.
	\newblock In {\em Advances in neural information processing systems}, pages
	164--171, 1993.
	
	\bibitem{dimopoulos_use_1995}
	Yannis Dimopoulos, Paul Bourret, and Sovan Lek.
	\newblock Use of some sensitivity criteria for choosing networks with good
	generalization ability.
	\newblock {\em Neural Processing Letters}, 2(6):1--4, 1995.
	
	\bibitem{dimopoulos_neural_1999}
	Ioannis Dimopoulos, J.~Chronopoulos, Aikaterini Chronopoulou-Sereli, and Sovan
	Lek.
	\newblock Neural network models to study relationships between lead
	concentration in grasses and permanent urban descriptors in {Athens} city
	({Greece}).
	\newblock {\em Ecological modelling}, 120(2-3):157--165, 1999.
	
	\bibitem{ruck_feature_1990}
	Dennis~W. Ruck, Steven~K. Rogers, and Matthew Kabrisky.
	\newblock Feature selection using a multilayer perceptron.
	\newblock {\em Journal of Neural Network Computing}, 2(2):40--48, 1990.
	
	\bibitem{bishop1995neural}
	Christopher~M Bishop et~al.
	\newblock {\em Neural networks for pattern recognition}.
	\newblock Oxford university press, 1995.
	
	\bibitem{lecun2012efficient}
	Yann~A LeCun, L{\'e}on Bottou, Genevieve~B Orr, and Klaus-Robert M{\"u}ller.
	\newblock Efficient backprop.
	\newblock In {\em Neural networks: Tricks of the trade}, pages 9--48. Springer,
	2012.
	
\end{thebibliography}







\pagebreak
\section*{Tables}
\begin{table}[h]
	\centering
	\begin{tabular}{ccccccccc}
		\hline
		\multirow{2}{*}{} & \multicolumn{2}{c}{test loss} & \multicolumn{2}{c}{test error (\%)} & \multicolumn{2}{c}{\# of variables} & \multicolumn{2}{c}{FDR} \\ \cline{2-9}
		& initial       & final         & initial           & final           & original        & significant       & estimated    & actual   \\ \hline
		dataset 1          & \begin{tabular}[c]{@{}c@{}}1.177e-2\\ (5.617e-3)\end{tabular}      & \begin{tabular}[c]{@{}c@{}}1.172e-2\\ (6.474e-3)\end{tabular}      & \begin{tabular}[c]{@{}c@{}}0.36\\ (0.17)\end{tabular}              & \begin{tabular}[c]{@{}c@{}}0.27\\ (0.10)\end{tabular}            & \begin{tabular}[c]{@{}c@{}}69.36\\ (5.07)\end{tabular}           & \begin{tabular}[c]{@{}c@{}}61.92\\ (2.48)\end{tabular}             & \begin{tabular}[c]{@{}c@{}}0.093\\ (0.004)\end{tabular}        & \begin{tabular}[c]{@{}c@{}}0.105\\ (0.044)\end{tabular}    \\ \hline
		dataset 2          & \begin{tabular}[c]{@{}c@{}}4.400e-4\\ (1.697e-3)\end{tabular}         & \begin{tabular}[c]{@{}c@{}}3.220e-5\\ (1.549e-4)\end{tabular}         & \begin{tabular}[c]{@{}c@{}}0.00\\ (0.00)\end{tabular}              & \begin{tabular}[c]{@{}c@{}}0.00\\ (0.00)\end{tabular}            & \begin{tabular}[c]{@{}c@{}}66.88\\ (8.73)\end{tabular}           & \begin{tabular}[c]{@{}c@{}}59.36\\ (5.87)\end{tabular}             & \begin{tabular}[c]{@{}c@{}}0.094\\ (0.005)\end{tabular}        & \begin{tabular}[c]{@{}c@{}}0.107\\ (0.057)\end{tabular}    \\ \hline
		dataset 3          & \begin{tabular}[c]{@{}c@{}}7.046e-1\\ (1.098e-2)\end{tabular}         & \begin{tabular}[c]{@{}c@{}}1.866e-2\\ (1.385e-2)\end{tabular}         & \begin{tabular}[c]{@{}c@{}}49.42\\ (1.69)\end{tabular}             & \begin{tabular}[c]{@{}c@{}}0.47\\ (0.48)\end{tabular}            & \begin{tabular}[c]{@{}c@{}}26.40\\ (13.68)\end{tabular}           & \begin{tabular}[c]{@{}c@{}}23.00\\ (11.87)\end{tabular}             & \begin{tabular}[c]{@{}c@{}}0.076\\ (0.031)\end{tabular}        & \begin{tabular}[c]{@{}c@{}}0.114\\ (0.089)\end{tabular}    \\ \hline
	\end{tabular}
	\caption{Summary statistics of variable selection on the simulation datasets $1 \sim 3$ (averaged over 25 simulations) when $p^\prime=64,\eta^*=0.1,\varepsilon=1$. The numbers in parentheses are corresponding standard deviations.}
	\label{tab:1}
\end{table}

\pagebreak
\section*{Figures}
\begin{figure}[h]
\centering
\includegraphics[width=0.5\linewidth]{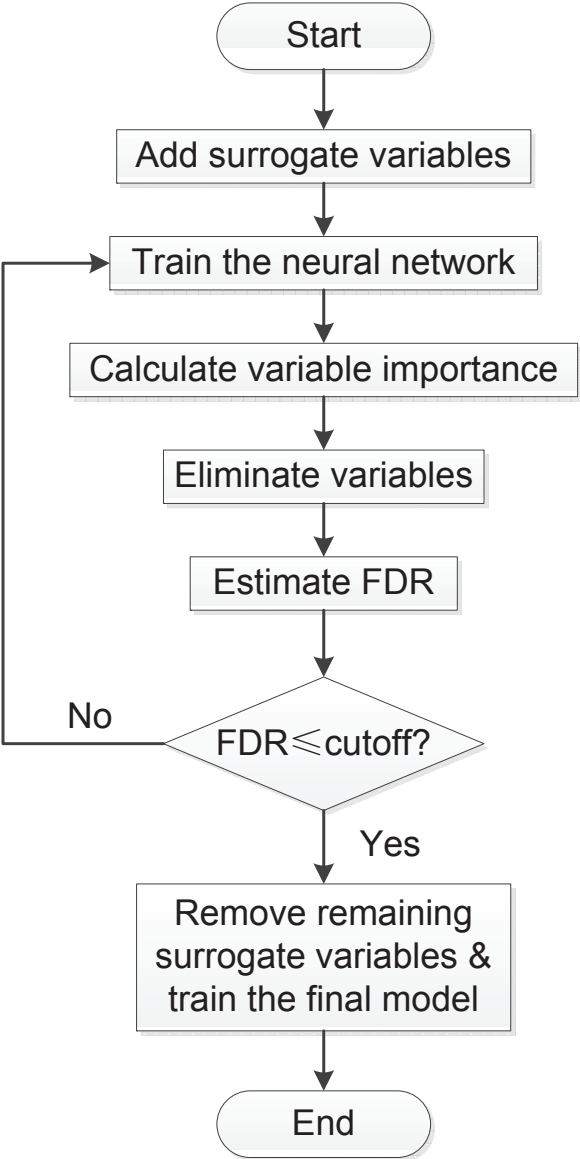}
\caption{Flowchart of SurvNet.}
\label{fig:1}
\end{figure}

\pagebreak
\begin{figure}[h]
\centering
\includegraphics[width=1\linewidth]{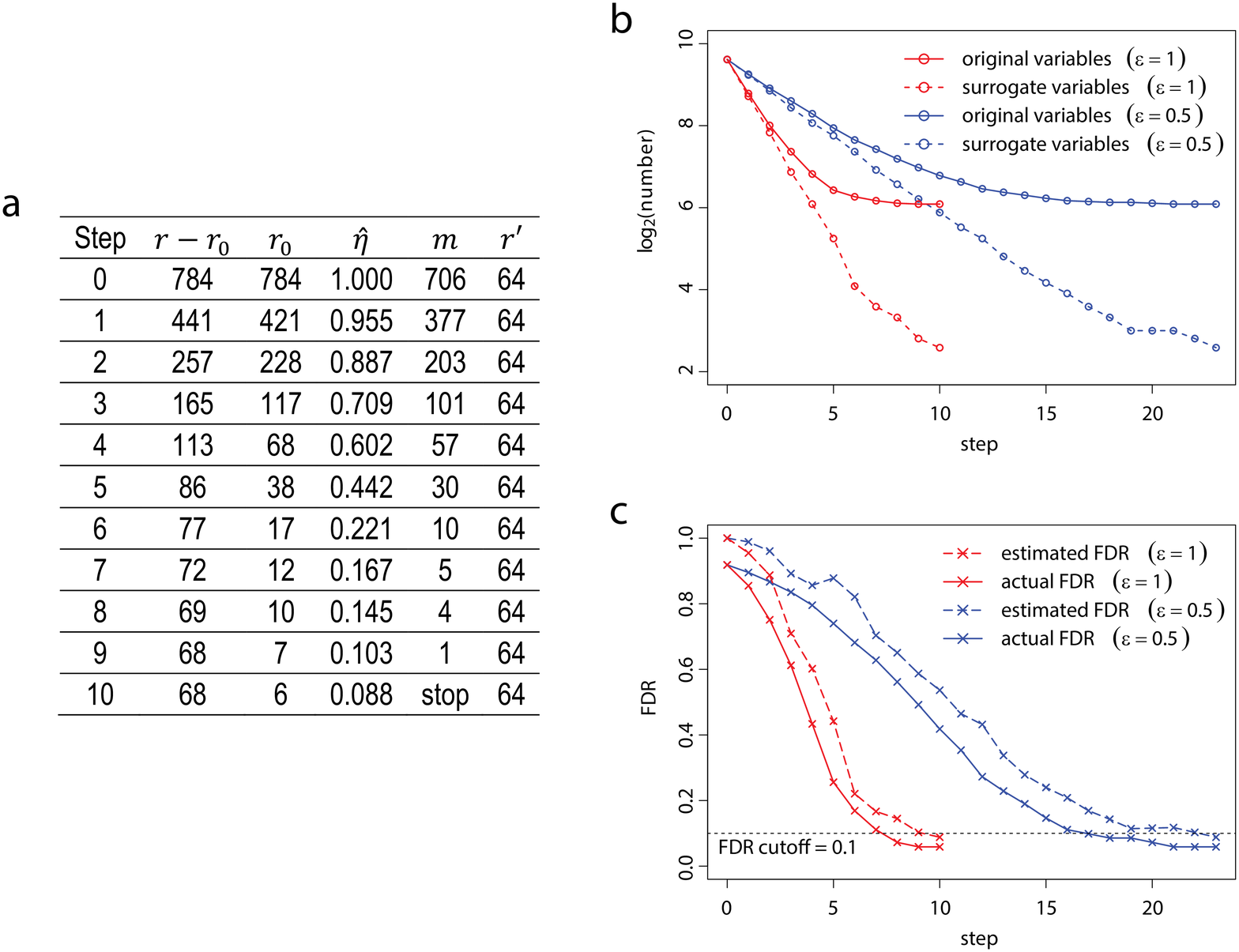}
\caption{Variable selection on one simulation dataset with independent variables. (a) The number of original variables ($r-r_0$), surrogate variables ($r_0$), and significant variables ($r^\prime$) left at each step of the selection process, together with the estimated FDR ($\hat{\eta}$) and the number of variables to be eliminated in the next step ($m$), when $p^\prime=64$, $\eta^*=0.1$, and $\varepsilon=1$. (b) The number of original and surrogate variables along the selection processes with different elimination rates when $p^\prime=64$ and $\eta^*=0.1$. (c) The estimated and actual value of FDR along the selection processes with different elimination rates when $p^\prime=64$ and $\eta^*=0.1$.}
\label{fig:2}
\end{figure}

\pagebreak
\begin{figure}[h]
\centering
\includegraphics[width=1\linewidth]{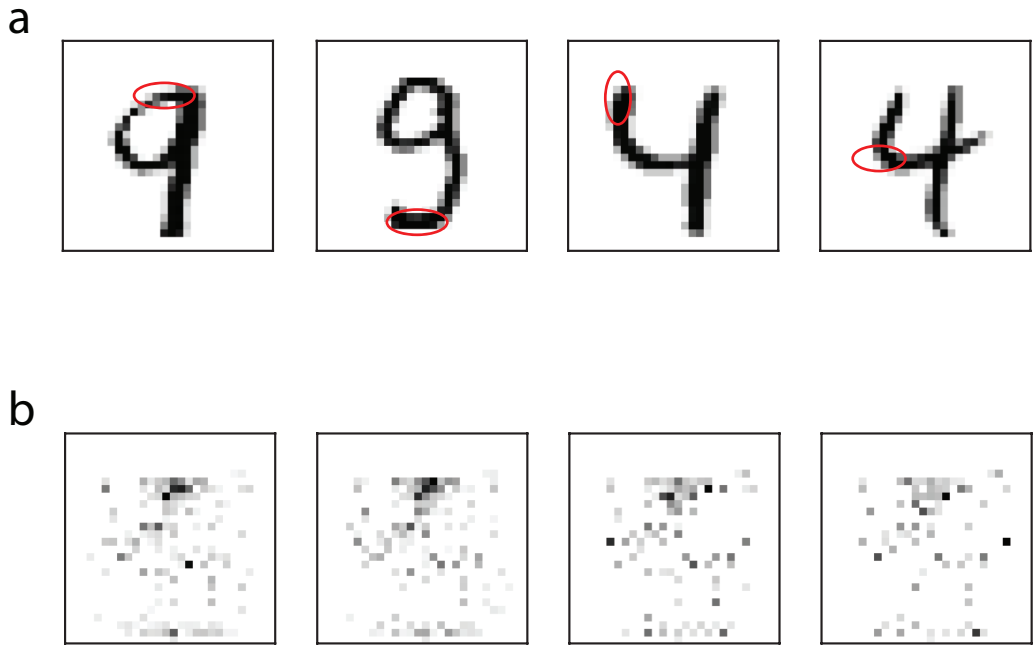}
\caption{Variable selection on the MNIST dataset of digit 4 and 9. (a) Examples of handwritten digits 4 and 9 (two images for each). The circles mark the locations of distinctive pixels of these two digits. (b) Heatmaps of the $28\times28$ pixels under four conditions with different FDR cutoffs and elimination rates, which display the relative importance of the remaining pixels. The darker the color of a pixel, the more important it is. The corresponding conditions are (from left to right): $\eta^*=0.1$, $\varepsilon=1$; $\eta^*=0.1$, $\varepsilon=0.5$; $\eta^*=0.01$, $\varepsilon=1$; $\eta^*=0.01$, $\varepsilon=0.5$.}
\label{fig:3}
\end{figure}

\pagebreak
\begin{figure}[h]
\centering
\includegraphics[width=1\linewidth]{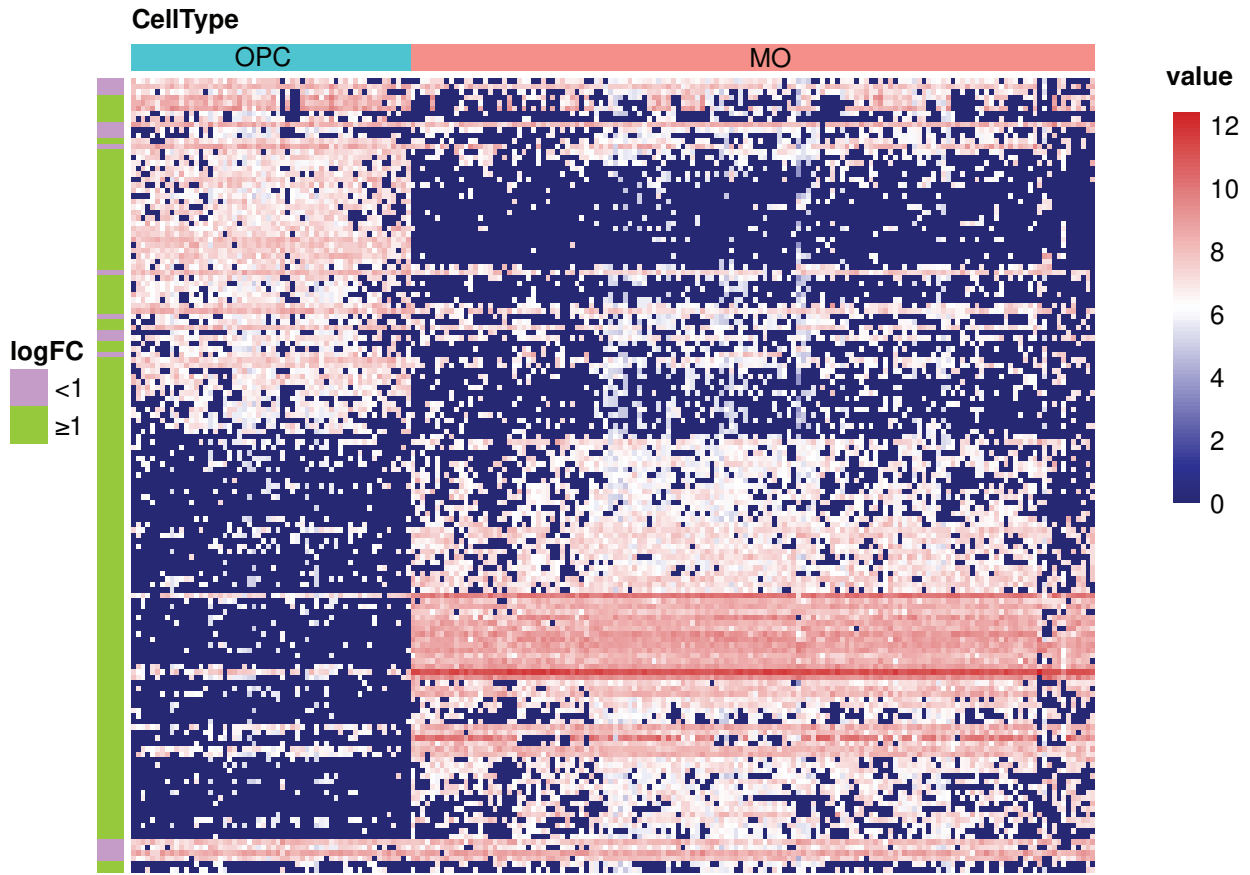}
\caption{Heatmap showing the expression of the selected genes in the single-cell RNA-Seq dataset in two groups of cells. Rows represent individual genes and columns are 200 randomly chosen cells. The genes whose log-fold changes in OPCs and MOs are less than 1 are distinguished from others.}
\label{fig:4}
\end{figure}



\end{document}